\documentclass[11pt]{tea}
\usepackage[utf8]{inputenc} % allow utf-8 input
\usepackage[T1]{fontenc}    % use 8-bit T1 fonts
\usepackage{hyperref}       % hyperlinks
\usepackage{url}            % simple URL typesetting
\usepackage{booktabs}       % professional-quality tables
\usepackage{booktabs}      
\usepackage{amsfonts}       % blackboard math symbols
\usepackage{nicefrac}       % compact symbols for 1/2, etc.
\usepackage{microtype}      % microtypography
\usepackage{xcolor}         % colors

\usepackage{algorithm}
\usepackage{algpseudocode}
\usepackage{amsmath}

\usepackage[table]{xcolor}

\usepackage{amsmath} % For math operators like \arg\max
\usepackage{booktabs} % Required for \toprule, \midrule, and \bottomrule
\usepackage{multirow} % Required to span text across multiple rows in the first column
\usepackage{graphicx} % Required for \resizebox to scale the tables to the text width

\usepackage{algorithm}
\usepackage{algpseudocode}
\usepackage{xcolor}
\usepackage{tikz}
\usetikzlibrary{fit,calc}

\definecolor{boxred}{RGB}{145,95,95}
\definecolor{boxgold}{RGB}{155,125,70}
\definecolor{boxblue}{RGB}{75,115,145}
\definecolor{boxgreen}{RGB}{80,125,90}

\usepackage{algorithm}
\usepackage{algpseudocode}
\usepackage{xcolor}

\definecolor{bavtred}{RGB}{140,95,95}
\definecolor{bavtgold}{RGB}{145,120,75}
\definecolor{bavtblue}{RGB}{80,115,145}
\definecolor{bavtgreen}{RGB}{85,125,95}

\newcommand{\bavtblock}[2]{%
  \Statex \vspace{0.3ex}
  \begin{center}
   \vspace{0.1cm}
    {\footnotesize\textcolor{#1}{\textbf{--- #2 ---}}}
    \vspace{0.1cm}
  \end{center}
}

\usepackage{amsthm}
\newtheorem{theorem}{Theorem}

\usepackage[dvipsnames]{xcolor}
\usepackage[most]{tcolorbox}
\usepackage{listings}
\usepackage{enumitem}

\definecolor{PromptBlue}{HTML}{1F4E79}
\definecolor{PromptBlueBg}{HTML}{EAF2FB}

\definecolor{PromptGreen}{HTML}{1A531A}  

\definecolor{PromptGreenBg}{HTML}{F0F9F0}

\tcbset{
  promptbox/.style={
    enhanced,
    breakable,
    colback=PromptBlueBg,
    colframe=PromptBlue,
    coltitle=white,
    fonttitle=\bfseries,
    boxed title style={colback=PromptBlue, colframe=PromptBlue},
    title={#1},
    sharp corners,
    boxrule=0.8pt,
    left=10pt,right=10pt,top=8pt,bottom=8pt,
    before skip=10pt, after skip=14pt,
  }
}

\newtcblisting{PromptBox}[1]{%
reset, 
    enhanced,
    breakable,
    colback=PromptGreenBg,
    colframe=PromptGreen,  
    coltitle=white,       
    fonttitle=\bfseries,
    title={#1},            
    sharp corners,
    boxrule=0.8pt,
    listing only,
    toprule=0.8pt, bottomrule=0.8pt, leftrule=0.8pt, rightrule=0.8pt, 
    titlerule=0.5pt,        
    left=10pt,right=10pt,top=8pt,bottom=8pt,
    before skip=10pt, after skip=14pt,
  }

\usepackage{soul} % for \ul (underline)
\newcount\Comments
\newcommand{\kibitz}[2]{\ifnum\Comments=1{\textcolor{#1}{\textsf{\footnotesize #2}}}\fi}

\newtheorem{assumption}{Assumption}
\newtheorem{remark}{Remark}

\title{Spend Less, Reason Better: \\Budget-Aware Value Tree Search for LLM Agents}

\author[1,2,\star]{Yushu Li}
\author[1,2,\star]{Wenlong Deng}
\author[1]{Jiajin Li}
\author[1,2,\dagger]{Xiaoxiao Li}

\affiliation[1]{University of British Columbia}
\affiliation[2]{Vector Institute}

\contribution{%
  $^\star$ Equal Contribution. \\
  $^\dagger$ Correspondence to Xiaoxiao Li at \href{mailto:xiaoxiao.li@ece.ubc.ca}{xiaoxiao.li@ece.ubc.ca}%
}

\abstract{
{ Abstract}: Test-time scaling has become a dominant paradigm for improving LLM agent reliability, yet current approaches treat compute as an abundant resource, allowing agents to exhaust token and tool budgets on redundant steps or dead-end trajectories. Existing budget-aware methods either require expensive fine-tuning or rely on coarse, trajectory-level heuristics that cannot intervene mid-execution. We propose the Budget-Aware Value Tree (BAVT), a training-free inference-time framework that models multi-hop reasoning as a dynamic search tree guided by \emph{step-level value estimation} within a single LLM backbone. Another key innovation is a \emph{budget-conditioned node selection mechanism} that uses the remaining resource ratio as a natural scaling exponent over node values, providing a principled, parameter-free transition from broad exploration to greedy exploitation as the budget depletes. To combat the well-known overconfidence of LLM self-evaluation, BAVT employs a residual value predictor that scores relative progress rather than absolute state quality, enabling reliable pruning of uninformative or redundant tool calls. We further provide a theoretical convergence guarantee, proving that BAVT reaches a terminal answer with probability at least $1-\epsilon$ under an explicit finite budget bound. Extensive evaluations on four multi-hop QA benchmarks across two model families demonstrate that BAVT consistently outperforms parallel sampling baselines. Most notably, BAVT under strict low-budget constraints surpasses baseline performance at $4\times$ the resource allocation, establishing that intelligent budget management fundamentally outperforms brute-force compute scaling.}

\begin{document}
\maketitle

\begin{figure*}[t]
    \centering
    \includegraphics[width=0.82\textwidth]{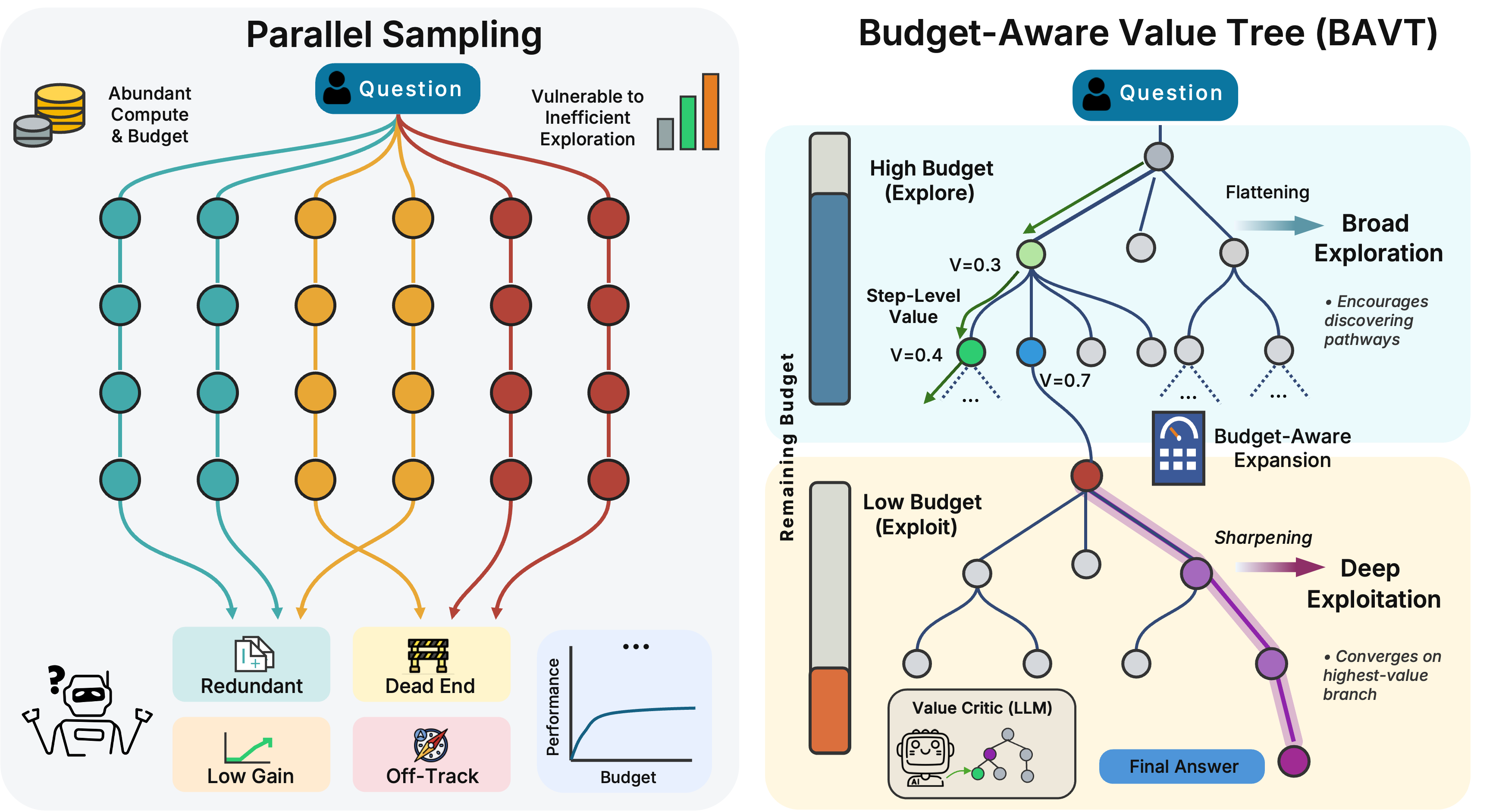}
    \caption{\textbf{Budget-Aware Value Tree (BAVT)} versus parallel sampling.
    \textbf{Left:} Parallel sampling explores many trajectories in parallel but may waste budget on redundant or dead-end paths.
    \textbf{Right:} BAVT performs tree-structured search with step-level value estimation and budget-aware expansion to adaptively shift from broad exploration (high remaining budget) to deep exploitation (low remaining budget), improving the performance--efficiency trade-off under strict resource constraints.}
    \label{fig:teaser_bavt}
    \vspace{-0.5em}
\end{figure*}

\iffalse

\fi

\section{Introduction}

The integration of external tools has transformed Large Language Models (LLMs) from passive text generators into autonomous agents capable of gathering information and executing tasks in complex environments~\citep{yao2022react,schick2023toolformer,deng2025grpo,luo2025large,jin2025search}. To improve reliability on challenging multi-hop reasoning tasks, recent work has increasingly relied on \emph{test-time scaling}~\citep{snell2025scaling,zhu2025scaling}---allocating additional computational resources during inference through reflection~\citep{shinn2023reflexion}, parallel sampling~\citep{wang2022self}, and search algorithms~\citep{yao2023tree}. Long-horizon systems~\citep{yang2024swe,openclaw_github_2026,openai_gpt53_codex_2026} exemplify this paradigm, running extended reasoning loops over hours or days. While increasing test-time compute generally improves task performance~\citep{snell2025scaling,zhu2025scaling}, a fundamental question remains underexplored:

\vspace{0.3cm}
\begin{center}
\begin{tcolorbox}[
    reset, 
    colback=gray!5, 
    colframe=black!60, 
    boxrule=0.5pt, 
    toprule=0.5pt, bottomrule=0.5pt, leftrule=0.5pt, rightrule=0.5pt, 
    sharp corners, 
    left=8pt, right=8pt, top=4pt, bottom=4pt, 
    width=0.75\textwidth
]
\centering\emph{How can autonomous agents achieve better task performance under a constrained compute budget?}
\end{tcolorbox}
\end{center}

In practice, current agents are designed to maximize accuracy under expanded or unrestricted compute budget, but rarely incorporate mechanisms for fine-grained budget control. Without structured resource management, agents frequently exhaust token limits and tool API calls on redundant or low-yield steps~\citep{cemri2025multi,lu2025exploring,kim2025cost}. What's worse, blindly allocating more resources often produces only marginal accuracy gains and yields diminishing returns after excessive spending ~\citep{liu2025budget}. Recent budget-aware methods have begun addressing this gap, but they face two critical limitations. First, approaches targeting general LLM reasoning~\citep{han2025token, li2025selfbudgeter} require expensive fine-tuning and do not transfer to autonomous agent workflows. Second, agent-specific frameworks like BATS~\citep{liu2025budget} incorporate the remaining budget into the prompt, but rely entirely on the LLM's implicit ability to self-regulate and manage budgets only at the trajectory level and lack provable certification of converging to a good solution. Critically, because these frameworks lack the ability to intervene at intermediate reasoning steps, they are fundamentally unable to detect and abandon failing trajectories in real time. As a consequence, agents routinely fall into dead ends or infinite loops, silently exhausting substantial budgets on unpromising directions before any corrective action can occur~\citep{cemri2025multi}. This absence of step-level budget-aware control represents a key barrier to deploying autonomous agents under real-world resource constraints.

To overcome these limitations, we propose the \textbf{Budget-Aware Value Tree (BAVT)}, a training-free, inference-time framework that unifies tree-structured search, step-level value estimation, and adaptive budget control within a single LLM backbone. BAVT models the reasoning process as a dynamic search tree, where nodes represent intermediate states and edges correspond to actions or tool invocations. This structure allows the agent to explore multiple candidate trajectories
instead of committing to a single linear path.  To guide the search, we introduce a \emph{step-level value critic} that evaluates the relative progress of each reasoning step. Unlike standard LLM self-evaluation, which tends toward overconfidence, our critic predicts residual value deltas, which score marginal information gain rather than absolute state quality, enabling reliable pruning of uninformative branches. Another key innovation of BAVT is a \emph{budget-conditioned node selection mechanism}: we transform node values into a sampling distribution using a power-based scaling function, where the exponent is the inverse of the remaining budget ratio. \emph{When budgets are abundant, the distribution promotes broad exploration; as the budget depletes, it sharpens to concentrate probability mass on the highest-value branches, providing a principled, parameter-free transition from exploration to exploitation}. We further provide a theoretical convergence guarantee, proving that BAVT reaches a terminal answer under an explicit finite budget bound.

Extensive evaluations on four multi-hop QA benchmarks across two model families and three budget tiers demonstrate that BAVT consistently outperforms parallel sampling baselines. Most strikingly, BAVT under strict low-budget constraints surpasses baseline performance at $4\times$ the resource allocation, establishing that intelligent budget management fundamentally outperforms brute-force compute scaling.

Our contributions are summarized as follows:
\begin{itemize}[leftmargin=*,nosep]
    \item \textbf{Budget-Aware Tree Search for Test-Time Scaling.} We formulate the problem of budget-aware agent testing-time scaling under strict token and tool-call constraints, and model the reasoning process as a dynamic search tree that enables fine-grained, step-level resource allocation.

    \item \textbf{BAVT: A Training-Free Framework with Theoretical Guarantees.} We propose the Budget-Aware Value Tree, featuring (i)~a residual value critic that scores relative progress to mitigate LLM overconfidence, and (ii)~a budget-conditioned node selection mechanism that provides a principled, parameter-free transition from exploration to exploitation as resources deplete. We prove that BAVT converges to a terminal answer under an explicit finite budget bound.

    \item \textbf{Spend Less, Reason Better.} Comprehensive evaluations across four multi-hop QA benchmarks using both instruct and reasoning models demonstrate that BAVT achieves a superior performance-efficiency trade-off at every budget level. Most notably, low-budget BAVT outperforms high-budget baselines, confirming that intelligent allocation outperforms brute-force scaling.
\end{itemize}

% ============================================================

\section{Related Work}

\subsection{Tool-Augmented LLM Agents}

The integration of external tools has significantly advanced the capabilities of Large Language Models (LLMs), transitioning them from static text generators to active agents capable of interacting with dynamic environments. Foundational frameworks such as ReAct \citep{yao2022react}, Toolformer \citep{schick2023toolformer}, and WebGPT \citep{nakano2021webgpt} have demonstrated the efficacy of interleaving reasoning traces with tool actions to solve complex queries. More recently, the development of robust orchestration frameworks, such as LangChain \citep{chase2022langchain}, and advanced agentic evaluation toolkits, like Inspect AI \citep{ukaisi2024inspect} and OctoTools \citep{lu2025octotools}, have standardized the deployment and testing of these complex multi-hop agents. While recent RL-based approaches attempt to optimize tool-use directly during training, they often suffer from severe instability and high computational overhead \citep{jin2025search, zhang2025criticsearch, sun2025zerosearch, deng2025grpo}. Consequently, deployed agents typically rely on naive autonomous loops that assume infinite resources, frequently trapping them in costly dead ends \citep{cemri2025multi, kim2025cost}.

\subsection{Test-Time Scaling}

To overcome the limitations of linear decoding, recent literature has systematically shifted toward test-time scaling—allocating more computational resources during inference to optimally improve reasoning robustness \citep{snell2025scaling}. Expanding upon these foundational scaling laws, recent works have specifically investigated scaling test-time compute for LLM agents \citep{zhu2025scaling}, emphasizing the unique performance gains achieved when autonomous systems are granted extended inference budgets for multi-step tool interactions. Methodologies such as Self-Consistency \citep{wang2022self}, Tree of Thoughts (ToT) \citep{yao2023tree}, Graph of Thoughts (GoT) \citep{besta2024graph}, and Language Agent Tree Search (LATS) \citep{zhou2023language} formulate the reasoning process as a search problem over a vast state space. Furthermore, recent advancements explore dynamically optimizing the geometry of these search spaces, such as adaptively balancing whether to expand wider or search deeper during inference \citep{inoue2025wider}. Drawing inspiration from reinforcement learning, recent works increasingly employ actor-critic paradigms during inference, while reflection mechanisms and prompt-based critics serve as value functions to evaluate intermediate states and facilitate self-correction \citep{shinn2023reflexion}. While these search-based and value-guided algorithms achieve strong performance, they are predominantly accuracy-driven and operate under the assumption of unbounded computational resources. They generally lack internal mechanisms to penalize expensive actions or adapt their search geometry based on resource depletion.

\subsection{Budget-Aware Inference}

As the economic and computational costs of deploying LLMs become a critical bottleneck, a growing subfield has focused on budget-aware inference. Initial strategies, such as model cascading \citep{chen2023frugalgpt} and routing systems like EcoAssistant \citep{zhang2023ecoassistant}, reduce costs by intelligently directing queries to cheaper models. Recent efforts in budget-aware inference have explored dynamic resource allocation for general LLM reasoning \citep{han2025token, li2025selfbudgeter}. However, these approaches are largely confined to static, closed-book problems rather than multi-hop agent workflows, and they typically rely on computationally expensive post-training to align the model's policy with resource constraints. However, autonomous agents pose a unique challenge due to the high financial cost of iterative environment interactions, such as web searches. To address this, recent works like the Budget-Aware Tool-Use (BATS) framework \citep{liu2025budget} explicitly impose limits on tool usage, while theoretical explorations identify performance phase transitions under strict multi-agent constraints \citep{liu2026phase}. Yet, these existing agent frameworks primarily manage resources using coarse heuristics or trajectory-level interventions, evaluating costs only after a full sequence fails or relying on rigid prompt-based warnings. In contrast, our proposed BAVT framework introduces fine-grained, step-level value evaluation, mathematically shifting the agent's strategy from broad exploration to greedy exploitation as the budget shrinks. 

\section{Methodology}

\subsection{Problem Formulation: Budget-Aware Agent Inference}

We study test-time scaling for tool-augmented LLM agents under a \emph{hard budget constraint}. Given a user question $x$, an agent interacts with external tools and performs intermediate reasoning steps before producing a final answer $\hat{y}$. The objective is to maximize answer correctness while strictly satisfying a predefined budget. We formalize this process as a resource-constrained deterministic decision process, defined by the tuple $(\mathcal{S}, \mathcal{A}, \mathcal{T}, \mathcal{B}, \mathcal{C})$.

\textbf{State Space ($\mathcal{S}$):} Let $s_t \in \mathcal{S}$ denote the state at step $t$. The state encapsulates the entire context available to the agent, including the initial user query, the history of prior actions, internal reasoning traces, and observations returned by external environments.

\textbf{Action Space ($\mathcal{A}$):} Let $a_t \in \mathcal{A}$ represent the action executed at step $t$. The action space encompasses both internal reasoning generations and external tool invocations.

\textbf{Transition Dynamics ($\mathcal{T}$):} The transition function $\mathcal{T}: \mathcal{S} \times \mathcal{A} \rightarrow \mathcal{S}$ dictates the environment dynamics. Given a state $s_t$ and an action $a_t$, the environment deterministically transitions to the next state $s_{t+1} = \mathcal{T}(s_t, a_t)$, appending the new action and its corresponding external observation to the context.

\textbf{Budget State Space ($\mathcal{B}$):} We formalize the budget not merely as a static constraint, but as a dynamic state space $\mathcal{B} \subseteq \mathbb{Z}_{\geq 0} \times \mathbb{Z}_{\geq 0}$ that tracks the remaining resources throughout the inference process. The process is strictly bounded by initial resource limits: the initial tool call budget $B_{\text{tool}}$ and the initial output token budget $B_{\text{token}}$. The available budget at any step $t$ is represented by the state variable $b_t = (b_{\text{tool}, t}, b_{\text{token}, t}) \in \mathcal{B}$, which is initialized as $b_0 = (B_{\text{tool}}, B_{\text{token}})$.

\textbf{Cost Function ($\mathcal{C}$):} Each action $a_t$ incurs a specific computational and financial cost defined by $\mathcal{C}(a_t) = (\mathcal{C}_{\text{tool}}(a_t), \mathcal{C}_{\text{token}}(a_t))$. The tool cost satisfies $\mathcal{C}_{\text{tool}}(a_t) = 1$ if $a_t$ is a successful tool call, and $0$ otherwise.
The token cost $\mathcal{C}_{\text{token}}(a_t)$ corresponds to the number of output tokens generated by the model for that step.

The remaining budget is updated iteratively after each step:
\begin{equation}
b_{t+1} = b_t - \mathcal{C}(a_t)
\end{equation}
which explicitly expands to the component-wise resource updates: 
% \Jiajin{What do you mean by "dual"?}
\begin{align}
b_{\text{tool}, t+1} &= b_{\text{tool}, t} - \mathcal{C}_{\text{tool}}(a_t) \nonumber \\
b_{\text{token}, t+1} &= b_{\text{token}, t} - \mathcal{C}_{\text{token}}(a_t) \label{eq:hard_budget}
\end{align}

\paragraph{Search Objective.}
We define a trajectory $\tau = (s_0, a_0, s_1, a_1, \dots, a_{T-1}, s_T)$ as a sequence of states and actions induced by the transition function $\mathcal{T}$. In this setting, test-time scaling corresponds to an informative search over the tree of states rooted at $s_0$. Instead of directly selecting a trajectory, the objective is to optimize a search policy $\pi$ that explores competing reasoning branches to identify a terminal answer $\hat{y}$ that maximizes expected correctness.

This problem setting motivates a search-based solution that performs fine-grained budget allocation across competing reasoning branches, rather than committing to a single linear trajectory.

\subsection{Overview of Budget-Aware Value Tree (BAVT)}

To effectively navigate the resource-constrained environments defined above, we introduce the Budget-Aware Value Tree (BAVT) framework, illustrated in Figure~\ref{fig:main_figure}. BAVT is a training-free architecture that fundamentally restructures agentic inference across three core pillars:

\begin{figure}[htbp]
  \centering
  \includegraphics[width=0.85\textwidth]{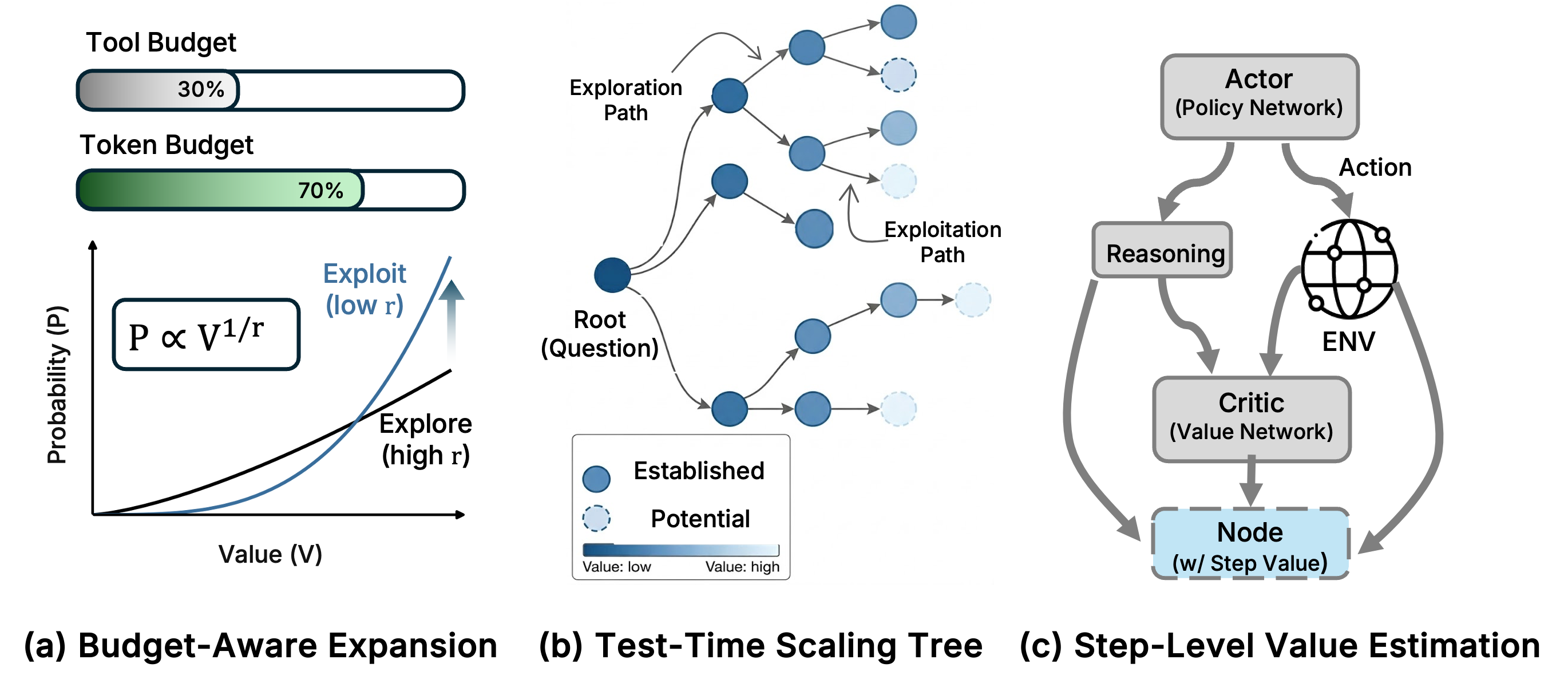}
  \caption{Overview of the Budget-Aware Value Tree (BAVT) framework. \textbf{a)} Budget-Aware Expansion dynamically adjusts the node selection distribution, shifting from exploration to exploitation as resources deplete. \textbf{b)} The Test-Time Scaling Tree models the reasoning process, allowing the agent to explore multiple paths. \textbf{c)} Step-Level Value Estimation utilizes a dual-role Actor-Critic setup within a single backbone to evaluate intermediate progress at the step level.}
  \label{fig:main_figure}
\end{figure}

\paragraph{1. Test-Time Scaling Tree.}
BAVT models the multi-hop reasoning process as a dynamic search tree designed specifically to scale compute at inference time. Within this structure, nodes represent intermediate reasoning states or external environmental observations, while edges correspond to the agent's generated actions. To populate this structure, we prompt the LLM backbone to act as a \textbf{Generator}, which observes the current state node and proposes a diverse set of potential next actions (e.g., tool calls or logical deductions). This tree-based formulation naturally facilitates test-time scaling by enabling the agent to safely branch out and explore multiple reasoning paths simultaneously without being trapped in a single dead-end trajectory.

\paragraph{2. Step-Level Value Estimation.}
To overcome the inefficiencies of delayed, trajectory-level evaluation, BAVT continuously assesses intermediate reasoning states immediately upon receiving environmental feedback. We dynamically prompt the same LLM backbone to alternate into a \textbf{Critic} role. This value-based evaluator assesses newly generated child nodes, computing a scalar value that estimates the state's distance to success. By acting as a dynamic proxy for how close the current reasoning trajectory is to a verifiable final answer, these step-level value estimations ground the tree expansion in immediate, objective information gain rather than speculative forward projection.

\paragraph{3. Budget-Aware Expansion.}
The critical mechanism that bridges the tree structure and the value estimations is our budget-aware node selection strategy. Rather than relying on static search heuristics, BAVT continuously monitors the remaining token and tool limits to mathematically adjust the probability distribution for selecting the next node to expand. This enforces a principled behavioral shift in the agent's policy: when resources are abundant, the framework encourages broad exploration of the search tree to identify promising directions; as the budget depletes, the framework aggressively transitions into greedy exploitation, forcing the agent to prioritize the highest-valued trajectory and synthesize a final answer before resources are completely exhausted.

\subsection{Step-Level Value Estimation}

\paragraph{Residual Value Prediction.}
A well-documented pathology in LLM-based self-evaluation is the tendency toward overconfidence, where models assign spuriously high absolute scores to mediocre or hallucinated reasoning steps. To mitigate this calibration failure, our critic evaluates relative progress rather than absolute state quality. Specifically, the critic predicts a residual score, or information delta ($\Delta_t$), reflecting the marginal utility of the most recent action. Let $V(n)$ denote the value of the parent node. The updated value for the newly generated child node $n'$ is computed as:
\begin{equation}
V(n') = \Phi(V(n) + \Delta_t)
\end{equation}
where $\Phi(\cdot)$ is a bounding function that restricts the value to a normalized range. By anchoring evaluations to relative deltas, the value function more reliably captures the true trajectory of reasoning progress and aggressively penalizes redundant or uninformative tool executions.

\paragraph{Value-Guided Step Instruction.}
These step-level values directly dictate the search tree's topological expansion. Let $\tau$ be a predefined confidence threshold for termination:
\begin{itemize}
    \item \textbf{Answer Generation} ($V(n') \ge \tau$): Sufficient evidence has been gathered, instructing the model to terminate the current path and synthesize a final answer.
    \item \textbf{Search Widening} ($V(n') \le V(n)$): The recent action yielded zero or negative information gain. To avoid stalled trajectories, the agent explores laterally by proposing divergent thoughts or tool calls.
    \item \textbf{Search Deepening} ($V(n) < V(n') < \tau$): The action yielded positive information gain but remains below the terminal threshold. The branch is promising, instructing the model to deepen the search via subsequent reasoning steps.
\end{itemize}
This structural guidance ensures the agent efficiently alternates between depth-first exploitation of promising chains and breadth-first exploration of uncertain pathways.

\subsection{Budget-Aware Node Expansion}

The core mechanism of the search is a dynamic, budget-aware node selection strategy. The search tree is initialized with the original question as the root node, assigned a minimum starting value. At each step, the framework samples a node to expand from the pool of existing candidate nodes, strictly excluding terminal answer nodes. 

\paragraph{Exploration to Exploitation Shift.}
Standard Upper Confidence Bound (UCB) formulas are well-suited for asymptotic exploration but fundamentally assume unbounded horizons, failing to adapt to strict, depleting resource constraints. To address this and prevent catastrophic budget exhaustion on unpromising paths, we introduce a budget-aware stochastic sampling mechanism. Let $r_t \in (0, 1]$ represent the effective remaining budget ratio at step $t$, defined as the minimum between the remaining tool budget ratio and the remaining token budget ratio:
\begin{equation}
r_t=\min\left(\frac{b_{\text{tool}, t}}{B_{\text{tool}}},\frac{b_{\text{token}, t}}{B_{\text{token}}}\right)
\end{equation}
We define a dynamic scaling exponent $\alpha_t$, which is inversely proportional to this limiting budget ratio:
\begin{equation}
\alpha_t=\frac{1}{r_t}
\end{equation}
For each candidate node $n_i \in \mathcal{N}_{\text{cand}}$ with an accumulated state value $V(n_i)$, we compute an unnormalized selection weight $w_{n_i}$ via a power-based scaling function:
\begin{equation}
    w_{n_i} = V(n_i)^{\alpha_t}
\end{equation}
The probability $\mathbb{P}(n_i)$ of selecting node $n_i$ for expansion is determined by normalizing these weights across all $N$ candidate nodes in the pool:
\begin{equation}
\label{eq: sampling}
  \mathbb{P}(n_i) = \frac{w_{n_i}}{\sum_{j=1}^{N} w_{n_j}}
\end{equation}

This formulation induces a budget-dependent shift in the agent’s behavior. 
When the budget is abundant ($r_t \approx 1$), $\alpha_t \approx 1$, yielding a sampling distribution roughly proportional to the raw node values and promoting exploration of the search space. 
As the budget decreases ($r_t \to 0$), $\alpha_t$ increases, magnifying value differences and concentrating probability mass on higher-valued nodes. 
In the limit, the distribution approaches a near-deterministic selection of the highest-valued node. 
Thus, the policy transitions from broad exploration in early stages to increasingly exploitative behavior as resources become scarce.

\paragraph{Node Expansion and Global Backpropagation.}

Once a node $n$ is selected, the generator synthesizes an action that interacts with the environment, producing a new observation and creating a corresponding child node $n'$. The prompt-based critic immediately evaluates this new state to assign an initial value $V(n')$. 
% \Jiajin{In the previous subsection, you typically use $i$ to denote the note. Please align.}

While the critic provides localized evaluations, the discovery of a terminal answer provides strong global signals about the validity of the preceding trajectory. Therefore, immediately after the first terminal answer node is generated, the framework triggers a global value update across the entire tree. For every subsequent step, the overall value $V(n)$ of a node $n$ is recursively updated bottom-up as follows:
\begin{equation}
    V(n) \leftarrow \frac{V(n) + \sum_{n_i \in \mathcal{N}_{\text{child}}(n)} V(n_i)}{1 + |\mathcal{N}_{\text{child}}(n)|}
\end{equation}
where $\mathcal{N}_{\text{child}}(n)$ represents the set of all generated child nodes for node $n$. In this update rule, the $V(n)$ term in the numerator relies on the node's original, unmodified critic evaluation. This operation smooths the initial local scores by incorporating the empirical success of downstream branches, ensuring that paths leading to multiple strong candidate answers are prioritized over isolated high-value nodes.

\subsection{Search Termination and Answer Extraction}

To maximize exploration and ensure the highest quality response, the search process generally avoids early termination upon discovering a single valid trajectory. Instead, the framework runs continuously to explore alternative reasoning paths until the hard budget constraints are reached.

\paragraph{Budget Backstop Mechanism.}
To guarantee a valid output under strict resource constraints and prevent catastrophic failure, we introduce a preemptive forced generation mechanism. During the search, if the agent fully exhausts the tool budget ($b_{\text{tool}, t} = 0$) or the remaining token budget ratio falls below a critical threshold $\eta$ (i.e., $\frac{b_{\text{token}, t}}{B_{\text{token}}} \le \eta$) without having successfully generated at least one proposed terminal answer (i.e., the answer set $\mathcal{A} = \emptyset$), the framework immediately halts standard expansion. It then selects the incomplete leaf node with the highest accumulated value:
\begin{equation}
    n^* = \arg\max_{n \in \mathcal{N}_{\text{leaf}}} V(n)
\end{equation}
From this optimal incomplete state $n^*$, the framework issues a deterministic instruction forcing the generator to synthesize a final answer based exclusively on the context and evidence gathered up to that point.

\paragraph{Final Answer Extraction.}
If the exhaustive search completes normally—meaning at least one valid trajectory successfully reached a terminal answer before the absolute budget limits ($b_{\text{tool}, t} = 0$ or $b_{\text{token}, t} = 0$) were fully depleted—the framework evaluates the proposed answer set $\mathcal{A}$. It selects the terminal node with the highest overall backpropagated value and extracts the final predicted answer $\hat{y}$ to be returned directly to the user.

\subsection{Theoretical Analysis: Convergence to Termination}

In this section, we provide a theoretical guarantee that, given a sufficiently large computational budget, the BAVT framework converges to a terminal answer by satisfying the generation condition ($V(s_t) \ge \tau$) with high probability. The result relies on three structural assumptions regarding the search space and the evaluator model.

\begin{assumption}[Strict Positive Progress] \label{assump:progress}
There exists at least one optimal ``oracle'' trajectory $\mathcal{P}^*$. Each deepening step along $\mathcal{P}^*$ yields a deterministic minimum information gain: $\Delta_t \ge \delta > 0$.
\end{assumption}

\begin{assumption}[Linearity before Threshold] \label{assump:linearity}
The bounding function $\Phi(\cdot)$ preserves the additive property $\Phi(x) = x$ for all $x \le v_{\textrm{max}}$, and the termination threshold satisfies $\tau < v_{\textrm{max}}$.
\end{assumption}

\begin{assumption}[Bounded Candidate Pool] \label{assump:bounded_pool}
The maximum number of active candidate nodes in the pool $\mathcal{N}_{\text{cand}}$ is strictly bounded by a constant $N_{\textrm{max}}$. Furthermore, all accumulated node values are bounded within $[v_{\textrm{min}}, v_{\textrm{max}}]$ (where $v_{\textrm{min}} > 0$), and the dynamic scaling exponent is bounded by a maximum value $\alpha_{\textrm{max}}$.
\end{assumption}

\begin{remark}
% These assumptions naturally align with the BAVT mechanism. Assumption~\ref{assump:progress} ensures that the oracle path is discoverable and makes measurable progress, effectively guaranteed by the search-widening pruning\dwl{section 3.3 we only use search-widening notaion, so maybe align them to avoid confusion } ($\Delta_t \le 0$). 

These assumptions naturally align with the BAVT mechanism. Assumption~\ref{assump:progress} ensures that the oracle path is discoverable and makes measurable progress, which is facilitated by the search widening mechanism ($\Delta_t \le 0$).
Assumption~\ref{assump:linearity} acts as a standard clipping mechanism. Assumption~\ref{assump:bounded_pool} prevents the infinite dilution of sampling probabilities as the tree expands (e.g., maintained via a priority queue).This assumption is mild in practice, as the candidate pool is naturally bounded by the limited number of sampled nodes.
% \dwl{This assumption is mild in practice, as the candidate pool is naturally bounded by the limited number of sampled nodes.} 
Consequently, the conditional probability of sampling any specific node in the candidate pool is strictly lower-bounded by $p_{\textrm{min}} = \frac{v_{\textrm{min}}^{\alpha_{\textrm{max}}}}{N_{\textrm{max}} \cdot v_{\textrm{max}}^{\alpha_{\textrm{max}}}} > 0$.
\end{remark}

\begin{theorem}[Probabilistic Convergence to Answer Generation]
Given Assumptions \ref{assump:progress}-\ref{assump:bounded_pool}, for any arbitrarily small failure probability $\epsilon > 0$, there exists a finite budget bound $B$ such that the BAVT framework successfully generates a node satisfying $V(s_t) \ge \tau$ with probability at least $1 - \epsilon$.
\end{theorem}

\begin{proof}
By Assumptions \ref{assump:progress} and \ref{assump:linearity}, reaching the threshold $\tau$ from the initial state $V(s_0)$ requires at most $K$ steps along $\mathcal{P}^*$, where:
\begin{equation}
    K = \left\lceil \frac{\tau - V(s_0)}{\delta} \right\rceil
\end{equation}
Completing $K$ steps guarantees $V(s_K) \ge V(s_0) + K\delta \ge \tau$, triggering termination deterministically. 

Let $c_{\max}$ be the maximum resource cost per expansion. An initial budget $B$ guarantees $M = \lfloor B / c_{\max} \rfloor$ valid expansion steps. Let $\mathcal{F}_{t-1}$ denote the filtration generated by the search history up to step $t-1$. Because the tree is built adaptively, we define the \emph{frontier node} $n^*_t$ of $\mathcal{P}^*$ as the deepest expanded node along $\mathcal{P}^*$ that resides in the active candidate pool at step $t$. 

Let $X_t \in \{0, 1\}$ be the indicator variable denoting whether this frontier node $n^*_t$ is successfully sampled for expansion at step $t$. By Assumption \ref{assump:bounded_pool} and the remark above, regardless of the search history, we have $\mathbb{P}(X_t = 1 \mid \mathcal{F}_{t-1}) \ge p_{\min} > 0$.

Because this conditional success probability is strictly bounded from below by a constant, we can construct an auxiliary sequence of i.i.d. random variables $Y_1, \dots, Y_M \sim \text{Bernoulli}(p_{\textrm{min}})$. Via a standard stochastic coupling argument, the cumulative sum $\sum_{t=1}^M X_t$ stochastically dominates $\sum_{t=1}^M Y_t$.

Consequently, the probability of failing to achieve the required $K$ successes along the oracle path is bounded by the lower tail of the binomial distribution of $Y_t$. Provided $M > K / p_{\textrm{min}}$, we apply the Chernoff bound. Let $x := M p_{\min}$, yielding:
\begin{equation}
    \mathbb{P}\left(\sum_{t=1}^M X_t < K\right) \le \mathbb{P}\left(\sum_{t=1}^M Y_t < K\right) \le \exp\left( - \frac{(x - K)^2}{2x} \right)
\end{equation}
To ensure this failure probability is at most $\epsilon$, it suffices to enforce $\exp\left( - \frac{(x - K)^2}{2x} \right) \le \epsilon$, which equates to $\frac{(x - K)^2}{2x} \ge \log\frac{1}{\epsilon}$. Solving for $x$, we require:
\begin{equation}
    x \ge K + \sqrt{2K\log\frac{1}{\epsilon}} + 2\log\frac{1}{\epsilon}
\end{equation}
Substituting $x = M p_{\min}$ back into the inequality, the required number of expansion steps $M$ is bounded by:
\begin{equation}
    M \ge \frac{1}{p_{\min}}\left( K + \sqrt{2K\log\frac{1}{\epsilon}} + 2\log\frac{1}{\epsilon}\right)
\end{equation}
Since $M = \lfloor B / c_{\max} \rfloor$, any initial budget $B \ge c_{\max} M$ explicitly guarantees $\mathbb{P}\left(\sum_{t=1}^M X_t < \left\lceil \frac{\tau - V(s_0)}{\delta} \right\rceil\right) \le \epsilon$. Hence, the framework reaches a termination node with probability at least $1 - \epsilon$.
\end{proof}

\section{Experiments}

\subsection{Experimental Setup}

\paragraph{Datasets.}
To rigorously evaluate our framework, we select four complex multi-hop reasoning benchmarks: HotpotQA~\citep{yang2018hotpotqa}, 2WikiMultihopQA~\citep{ho2020constructing}, MuSiQue~\citep{trivedi2022musique}, and Bamboogle~\citep{press2023measuring}. These datasets are specifically chosen because they strictly require sequential tool use and dynamic information gathering. The questions within these benchmarks are designed to be unanswerable relying solely on the internal parametric weights of the model, thereby serving as an ideal testbed for evaluating resource-constrained environment interactions.

\paragraph{Models Evaluated.}
To demonstrate that our framework is versatile and model-agnostic, we conduct our evaluations across two distinct model architectures:
\begin{itemize}
    \item \textbf{GPT-OSS-20B~\citep{agarwal2025gpt}:} A specialized reasoning model featuring robust internal chain-of-thought and native tool-use capabilities.
    \item \textbf{Qwen3-30B-A3B-Instruct-2507~\citep{yang2025qwen3}:} A high-capacity instruction-following model optimized for general-purpose tasks. 
\end{itemize}

\paragraph{Baseline Formulation.}
We compare our approach against a parallel sampling baseline, formulated as budget-constrained majority voting. To ensure a strictly fair comparison under identical resource constraints, this baseline executes $K$ independent reasoning trajectories concurrently. Let $\mathcal{T} = \{\tau_1, \tau_2, \dots, \tau_K\}$ denote the set of successfully completed trajectories, where each trajectory $\tau_i$ incurs a resource cost $\mathcal{C}(\tau_i)$, measured in tool calls and generated tokens. The generation process continues, maximizing the total number of trajectories $K$, until the cumulative cost exhausts the predefined budget limit $B$:
\begin{equation}
    \sum_{i=1}^{K} \mathcal{C}(\tau_i) \le B
\end{equation}
Once the budget $B$ is fully consumed, let $A(\tau_i)$ denote the terminal answer extracted from trajectory $\tau_i$. The final predicted answer $\hat{y}$ is determined by applying a majority voting function across all valid trajectories, formally defined as:
\begin{equation}
    \hat{y} = \arg\max_{y} \sum_{i=1}^{K} \mathbb{I}(A(\tau_i) = y)
\end{equation}
where $\mathbb{I}(\cdot)$ is the indicator function. This formalizes a rigorous scaling baseline that natively consumes the exact same computational budget as our proposed dynamic tree search.

\paragraph{Budget Configurations.}
To systematically analyze the performance-efficiency trade-off, we define three distinct budget levels for our evaluations. Each tier specifies a strict maximum limit for both the tool call budget ($B_{\text{tool}}$) and the token generation budget ($B_{\text{token}}$). Because reasoning models produce longer internal traces, their token budgets are scaled accordingly:
\begin{itemize}
    \item \textbf{Low Budget:} 5 maximum tool calls. Token limit is 2000 for the reasoning model and 1000 for the instruct model.
    \item \textbf{Middle Budget:} 10 maximum tool calls. Token limit is 4000 for the reasoning model and 2000 for the instruct model.
    \item \textbf{High Budget:} 20 maximum tool calls. Token limit is 8000 for the reasoning model and 4000 for the instruct model.
\end{itemize}

\paragraph{Budget-Aware Planning.}
To ensure the search trajectory respects the resource constraints of the environment, we execute a single, high-level planning step at the root node, as detailed in Section~\ref{subsec:planner_prompt}. We inject a budget hint directly into the system prompt, instructing the model to generate a strictly abstract, fact-free outline alongside an estimation of required tool calls. This plan is generated only once and is appended to the shared context of all subsequent nodes.

\paragraph{Implementation Details.}
We implement our method within the Inspect AI~\citep{ukaisi2024inspect} framework. For retrieval-augmented tool executions, we follow Search-R1~\citep{jin2025search}, utilizing the 2018 Wikipedia dump~\citep{karpukhin2020dense} as the foundational knowledge corpus. We employ the E5 model~\citep{wang2022text} as our dense retriever, fixing the number of retrieved passages to exactly five for each search query. Finally, to optimize generation dynamics, we apply the default hyperparameters tailored to each specific architecture according to their official guidelines. Specifically, for the Qwen instruct model, we utilize its default settings: a temperature of 0.7, a top-$p$ of 0.8, and a top-$k$ of 20. Conversely, for the GPT-OSS reasoning model, we adopt its respective defaults, setting both the temperature and top-$p$ to 1.0, while setting top-$k$ to 0.

\subsection{Main Results}

We present the empirical results of our framework compared against the parallel sampling baseline across the three predefined budget tiers. Figure~\ref{fig:oss_results} details the performance of the OSS-20B reasoning model, while Figure~\ref{fig:qwen_results} details the performance of the Qwen3-30B instruct model. Performance is measured using Exact Match (EM) and F1 scores.

\begin{figure}[tbp]
    \centering
    \includegraphics[width=\linewidth]{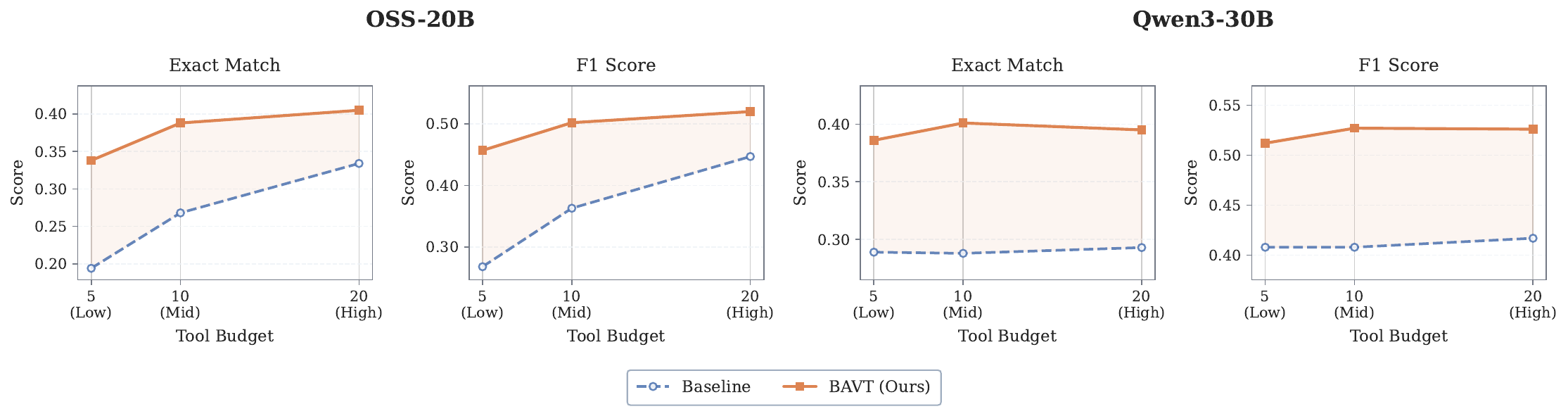}
    \vspace{-0.5cm}
    \caption{Average performance-efficiency trade-off across the four evaluated multi-hop QA benchmarks for OSS-20B and Qwen3-30B. BAVT operating under strict \textit{Low} budget constraints (5 calls) consistently rivals or surpasses the baseline's \textit{High} budget performance (20 calls), demonstrating that intelligent resource management fundamentally outperforms $4\times$ brute-force compute scaling.}

    \label{fig:budget_comparison}
\end{figure}

\paragraph{Spend Less, Reason Better.}
Across all four multi-hop datasets, the proposed BAVT framework consistently outperforms the parallel sampling baseline under identical constraints. Figure~\ref{fig:budget_comparison} illustrates the average scaling dynamics across both model architectures over the four multi-hop QA benchmarks. Across both of these distinct scaling behaviors, BAVT maintains a strictly superior performance frontier at every budget increment. Most notably, BAVT under strict low-budget constraints surpasses baseline performance at $4\times$ the resource allocation. For example, using the OSS-20B model, BAVT achieves a 0.338 average Exact Match (EM) at the \textit{Low} tier (5 tool calls), eclipsing the baseline's \textit{High} tier peak of 0.334 EM (20 tool calls). This explicitly establishes that intelligent budget management fundamentally outperforms brute-force compute scaling. By consistently achieving higher accuracy with fewer resources, the trends visually reinforce that fine-grained, value-guided tool allocation dominates unguided test-time scaling—allowing the system to reason better while spending significantly less.

\begin{figure}[h]
    \centering
    \includegraphics[width=\textwidth]{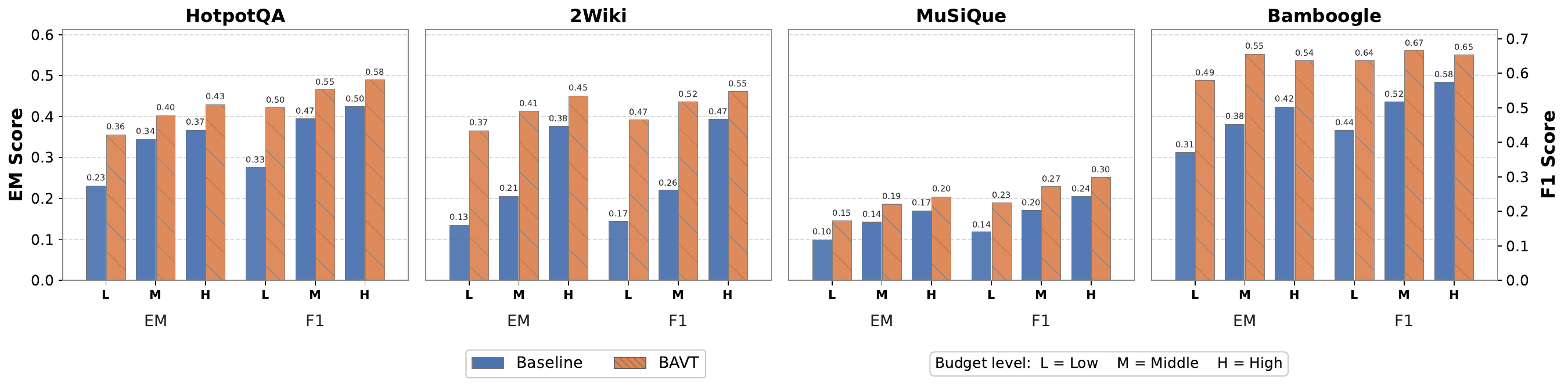}
    \vspace{-0.5cm}
    \caption{Performance of the OSS-20B reasoning model on multi-hop QA benchmarks. BAVT achieves strictly superior performance and resource efficiency compared to the baseline across all datasets and budget constraints.}
    \label{fig:oss_results}
\end{figure}

\paragraph{Reasoning Model Amplification.} 
For the OSS-20B model (Figure~\ref{fig:oss_results}), baseline performance scales effectively with increased budgets (average EM rises from 0.194 to 0.334). This validates that parallel majority voting benefits models equipped with multi-stage internal reasoning traces. Because these reasoning processes introduce high variance across different trajectories, scaling up rollouts helps identify a robust consensus solution. However, a major vulnerability remains: once a reasoning model generates a flawed intermediate premise, it tends to confidently justify the error, wasting substantial token and tool budgets exploring unpromising directions. The baseline algorithm is completely blind to this, allowing doomed trajectories to run to completion. 

BAVT, by contrast, acts as a dynamic regularizer. The step-level critic identifies factual drifts immediately after an environment interaction, while the budget-aware exponent ($\alpha_t$) dynamically truncates exploratory rabbit holes to force the exploitation of the most factually grounded path. By pruning errors before they compound, BAVT serves as a powerful performance multiplier. Across all three budget tiers, BAVT consistently outperforms the baseline by significant margins. This advantage is most pronounced under tight constraints: at the Low budget tier, BAVT elevates the average EM from 0.194 to 0.338. Remarkably, this surpasses the baseline's High budget performance (0.334 EM) while utilizing only one-quarter of the computational resources. This conclusively demonstrates that BAVT successfully balances exploration and exploitation, preventing resource waste on dead ends even when the available budget is severely restricted.

\begin{figure}[h]
    \centering
    \includegraphics[width=\textwidth]{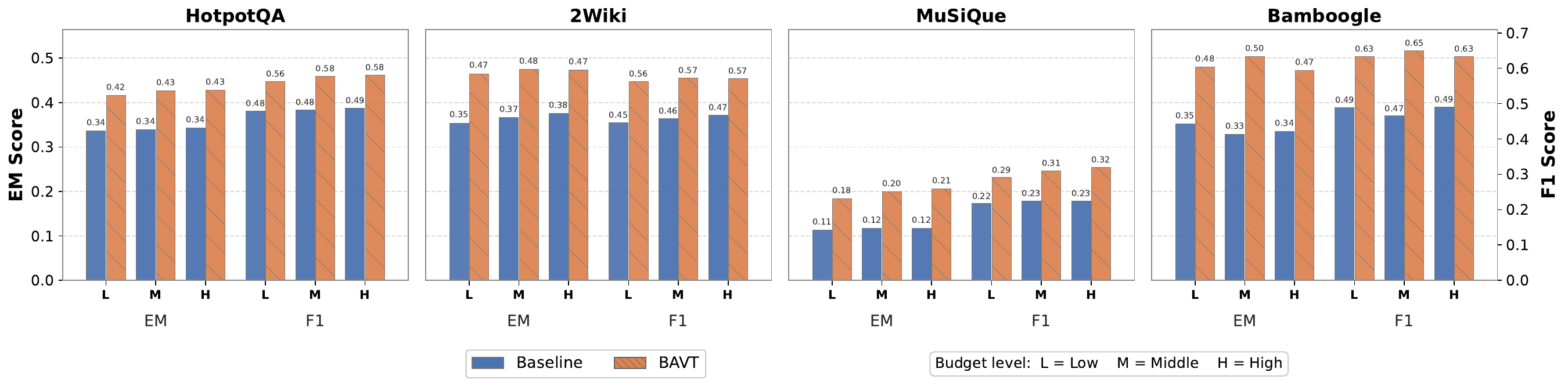}
    \vspace{-0.5cm}
    \caption{Performance of the Qwen3-30B instruct model across multi-hop QA benchmarks. BAVT achieves strictly superior performance, effectively raising the baseline’s ceiling.}
    \label{fig:qwen_results}
\end{figure}

\paragraph{Instruct Model Plateau.} 
Conversely, as shown in Figure~\ref{fig:qwen_results}, the baseline methodology largely fails to translate increased computational budgets into meaningful performance gains for the Qwen3-30B instruct model. The baseline average EM stagnates entirely, moving only from 0.289 (Low) to 0.293 (High). This plateau is fundamentally driven by the low generation variance and inherent overconfidence of non-reasoning architectures. Because instruction-tuned models are optimized to output definitive responses in a single pass, they suffer from severely limited generation diversity when confronted with multi-hop ambiguity. Consequently, merely increasing the parallel rollout budget results in mode collapse; the model exhaustively repeats the exact same failure trajectories rather than exploring alternative hypotheses. BAVT systematically overcomes this architectural limitation through its value-guided structural instructions. When a trajectory yields zero information gain, BAVT's "search widening" mechanism explicitly forces the instruct model to break out of its mode collapse and propose lateral, divergent tool calls. By artificially inducing the systematic exploration that instruct models lack natively, BAVT successfully breaks the baseline performance ceiling, achieving a dominant 0.386 average EM even at the Low budget tier.

However, it is also crucial to note that while BAVT significantly elevates the overall performance floor, its own scaling curve eventually flattens at higher budgets. This highlights a fundamental capacity bottleneck in instruct models: while BAVT can successfully orchestrate the search tree and retrieve the necessary evidence through forced exploration, the base model eventually lacks the complex reasoning capacity to synthesize massive amounts of accumulated multi-document context into a correct answer. Unlike reasoning models that utilize long internal traces to iteratively resolve ambiguity, instruct models reach a hard parameter-bound limit where simply providing more budget and retrieved context no longer translates into task success.

\paragraph{Dataset-Specific Observations.}
The performance gap is most pronounced on highly complex datasets requiring rigorous cross-document reasoning. On MuSiQue, standard linear trajectories are highly vulnerable to dense retrieval noise, causing the Qwen3-30B baseline to flatline at 0.12 EM across higher budgets. BAVT's step-level verification effectively filters this noise, pushing the EM to 0.21. On 2Wiki, BAVT demonstrates profound resource efficiency: the OSS-20B model achieves 0.37 EM under strict Low-budget constraints, nearly matching the baseline's High-budget peak (0.38 EM). This confirms that actively verifying intermediate sub-claims is vastly more compute-efficient than relying on unguided parallel rollouts. Finally, on Bamboogle, BAVT yields massive absolute improvements even at the Low tier (0.31 $\rightarrow$ 0.49 EM for OSS-20B). However, performance plateaus at the Middle tier (0.55 EM) as the budget scales. Because Bamboogle features relatively simpler queries, BAVT extracts the model's maximum reasoning capacity almost immediately, rendering further resource expenditure unnecessary.

\subsection{Ablation Studies}

\begin{table}[h]
\centering
\caption{Ablation study isolating the impact of the Tree Structure, Step-Level Value, and Budget-Aware Node Selection using the OSS-20B model at the Middle budget tier. Checkmarks (\checkmark) indicate an active component, while crosses ($\times$) indicate an inactive component.}
\label{tab:ablation}
\resizebox{0.85\textwidth}{!}{%
\begin{tabular}{ccc|cc|cc|cc|cc|cc}
\toprule
\multicolumn{3}{c|}{\textbf{Components}} & \multicolumn{2}{c|}{\textbf{HotpotQA}} & \multicolumn{2}{c|}{\textbf{2Wiki}} & \multicolumn{2}{c|}{\textbf{MuSiQue}} & \multicolumn{2}{c|}{\textbf{Bamboogle}} & \multicolumn{2}{c}{\textbf{AVG}} \\
\textbf{Tree} & \textbf{Value} & \textbf{Budget} & \textbf{EM} & \textbf{F1} & \textbf{EM} & \textbf{F1} & \textbf{EM} & \textbf{F1} & \textbf{EM} & \textbf{F1} & \textbf{EM} & \textbf{F1} \\
\midrule
$\times$ & $\times$ & $\times$ & 0.344 & 0.469 & 0.205 & 0.261 & 0.142 & 0.203 & 0.381 & 0.517 & 0.268 & 0.363 \\
\midrule
\checkmark & $\times$ & $\times$ & 0.243 & 0.340 & 0.124 & 0.163 & 0.100 & 0.140 & 0.392 & 0.478 & 0.215 & 0.280 \\
\checkmark & \checkmark & $\times$ & 0.356 & 0.481 & 0.265 & 0.332 & 0.158 & 0.224 & 0.456 & 0.575 & 0.309 & 0.403 \\
\midrule
\checkmark & \checkmark & \checkmark & \textbf{0.402} & \textbf{0.552} & \textbf{0.413} & \textbf{0.517} & \textbf{0.186} & \textbf{0.272} & \textbf{0.552} & \textbf{0.666} & \textbf{0.388} & \textbf{0.502} \\
\bottomrule
\end{tabular}%
}
\end{table}

To rigorously isolate the contributions of our proposed architecture, we conduct an incremental ablation study using the OSS-20B reasoning model at the Middle budget tier. We deconstruct the full BAVT framework into its three foundational components: Tree Structure, Step-Level Value, and Budget-Aware Node Selection. As detailed in Table~\ref{tab:ablation}, we evaluate configurations ranging from the standard parallel sampling baseline (where all structural components are disabled) to the full BAVT framework.

\paragraph{The Insufficiency of Random Tree Search.}
To isolate the effect of the tree structure itself, we test a basic configuration where the next expansion node is sampled uniformly at random from the candidate pool. As observed in Table~\ref{tab:ablation}, simply organizing the reasoning process into a search tree without intelligent guidance actively degrades performance. This tree-only configuration averages an EM of 0.215, significantly underperforming the baseline (0.268). In a vast, multi-hop search space, unguided tree expansion fragments the limited computational budget across less promising paths and dead ends, preventing the model from deeply investigating high-value trajectories.

\paragraph{The Impact of Step-Level Value.}
Integrating the Step-Level Value verifier yields a dramatic improvement, validating our step-wise value estimation strategy. Equipped with fine-grained evaluations for each reasoning state, the generator is directly guided by critic feedback, enabling the framework to reliably identify and prioritize promising directions for exploration via Equation~\ref{eq: sampling}. Activating this component alone elevates the average EM to 0.309, comfortably surpassing both the random tree configuration and the standard parallel baseline. However, in the absence of explicit budget-aware control, specifically lacking the dynamic scaling exponent $\alpha_t$, the framework remains constrained. It cannot systematically shift its strategy from broad exploration to greedy exploitation as computational resources deplete.

\paragraph{The Necessity of Budget-Aware Node Selection.}
While static step-level values reliably identify high-quality reasoning states, they fail to account for environmental resource constraints. Guided solely by static values, a reasoning model tends to endlessly deepen promising paths without widely explore potential paths. Activating Budget-Aware Node Selection ($\alpha_t$) resolves this critical bottleneck. By inversely tying the selection exponent to the remaining budget, BAVT forces the agent to aggressively transition into greedy exploitation as resources deplete. This final component elevates the framework to its peak performance (0.388 AVG EM), ensuring that high-quality trajectories identified by the verifier are successfully driven to completion rather than prematurely truncated by budget exhaustion.

\section{Conclusion}

In this work, we formalized the critical challenge of resource allocation in autonomous agent inference and proposed the Budget-Aware Value Tree (BAVT) to address the severe inefficiencies of unconstrained test-time scaling. By modeling multi-hop reasoning as a dynamic search tree guided by an LLM-driven step-level critic, BAVT introduces a principled node selection mechanism that seamlessly transitions the agent's policy from broad exploration to greedy exploitation as computational budgets deplete. Extensive empirical evaluations across four multi-hop QA benchmarks demonstrate that this training-free framework consistently achieves a vastly superior performance-efficiency trade-off compared to robust parallel sampling baselines. Most notably, BAVT operating under strict low-budget constraints frequently surpasses the high-budget performance of standard methods, successfully mitigating compounding errors in reasoning models and breaking the mode-collapse plateau inherent to instruction-tuned architectures. Ultimately, BAVT establishes a highly effective and adaptable paradigm for maximizing the reliability of autonomous agents under the strict, real-world resource constraints necessary for practical deployment.

\section{Limitations and Future Work}

While the Budget-Aware Value Tree (BAVT) demonstrates substantial improvements in resource-efficient reasoning, our framework has several limitations that present promising avenues for future research.

\paragraph{Inference Overhead of the Critic.} 
While BAVT significantly reduces redundant tool executions and prevents runaway generation trajectories, the dual-role prompting mechanism introduces inherent inference overhead. Evaluating the progress of every intermediate step using the main LLM backbone consumes a portion of the overall token budget. Although this trade-off is empirically net-positive for complex multi-hop tasks, future work should explore training lightweight, specialized Process Reward Models (PRMs) or training a dedicated value head directly on top of the base model. This would replace the prompt-based critic, substantially reducing the token footprint and latency associated with step-level verification.

\paragraph{Heterogeneous Tools and Asymmetric Costs.} 
Our current problem formulation and empirical evaluations primarily focus on a single external tool (web search) assigned a uniform, discrete cost ($\mathcal{C}_{\text{tool}} = 1$). In real-world deployments, agents must orchestrate a diverse arsenal of tools, such as code interpreters, database queries, and specialized APIs—that incur vastly asymmetric financial costs, execution latencies, and rate limits. A critical next step is extending the budget-aware node selection mechanism to incorporate dynamic, multi-dimensional pricing matrices. This would require the agent to learn complex trade-offs, such as balancing cheap, low-fidelity heuristic tools against expensive, highly accurate deterministic APIs under strict resource constraints.

\paragraph{Extension to Long-Horizon Agent Tasks.} 
The scope of our evaluation is currently centered on knowledge-intensive multi-hop question answering. However, autonomous agents are increasingly deployed in open-ended, long-horizon interactive environments. Adapting BAVT to handle tasks such as complex browser manipulation (e.g. Browsecomp~\citep{wei2025browsecomp}) and computer control benchmarks (e.g., OSWorld~\citep{xie2024osworld} or WebArena~\citep{zhou2023webarena}) is a natural progression. These environments often feature irreversible actions, partial observability, and highly delayed rewards, which will require extending our step-level value function to handle much more nuanced temporal credit assignment and state-space exploration.

\bibliographystyle{plainnat}
\bibliography{ref}

\newpage

\appendix

\section{BAVT Algorithm}

The complete execution flow of Budget-Aware Value Tree (BAVT) is detailed in Algorithm~\ref{alg:vbts}. Initialized with a root query and strict computational constraints, the framework iteratively constructs a search tree. At each step, a candidate node is sampled from a budget-annealed distribution, where the scaling exponent $\alpha_t$ smoothly transitions the agent's policy from broad exploration to greedy exploitation as resources deplete. Following node selection, the framework issues a deterministic, value-guided instruction (widen, deepen, or generate answer) to orchestrate the generator's action. The resulting state is then evaluated by the prompt-based critic to predict a residual information delta ($\Delta_t$), updating the node's accumulated value. The discovery of any terminal candidate answer immediately triggers a global, bottom-up value backpropagation to smooth local evaluations. This exhaustive search terminates strictly upon budget exhaustion, returning the extracted answer from the optimal leaf node.

\begin{algorithm}[hb]

\caption{Budget-Aware Value Tree (BAVT)}
\label{alg:vbts}
\begin{algorithmic}[1]
\Require Query $x$, Initial budgets $B_{\text{tool}}, B_{\text{token}}$
\Ensure Final predicted answer $\hat{y}$

\State Init: root $s_0 \leftarrow x$, budget $b_0 \leftarrow (B_{\text{tool}}, B_{\text{token}})$, $\mathcal{T}, \mathcal{N}_{\text{cand}} \leftarrow \{s_0\}$, $\mathcal{A} \leftarrow \emptyset$

\While{$b_{\text{tool}, t} > 0 \land b_{\text{token}, t} > 0$}
    \bavtblock{bavtred}{1. Budget-Aware Node Selection}
    \State $r_t \leftarrow \min(b_{\text{tool}, t}/B_{\text{tool}}, b_{\text{token}, t}/B_{\text{token}})$; $\alpha_t \leftarrow 1/r_t$
    \State Sample expansion node $n \sim P(n_i) \propto V(n_i)^{\alpha_t}$ from $\mathcal{N}_{\text{cand}}$

    \bavtblock{bavtgold}{2. Action Generation and Transition}
    \State Issue structural instruction $I_t$ based on $V(n)$ \Comment{Widen, Deepen, or Answer}
    \State Generate action $a_t \sim P(\cdot | n, I_t)$, observe transition $\mathcal{T}(n, a_t) \rightarrow n'$
    \State Update budget $b_{t+1} \leftarrow b_t - \mathcal{C}(a_t)$; append $n'$ to $\mathcal{T}$

    \bavtblock{bavtblue}{3. Value Critic and Backpropagation}
    \If{$n'$ is a terminal answer node}
        \State $V(n') \leftarrow V(n)$; $\mathcal{A} \leftarrow \mathcal{A} \cup \{n'\}$ \Comment{Inherit value, skip candidate pool}
    \Else
        \State Predict residual $\Delta_t$ for $n'$; update value $V(n') \leftarrow \Phi(V(n) + \Delta_t)$
        \State $\mathcal{N}_{\text{cand}} \leftarrow \mathcal{N}_{\text{cand}} \cup \{n'\}$
    \EndIf
    \If{$\mathcal{A} \neq \emptyset$}
        Backpropagate values bottom-up to update $V(n_k), \forall n_k \in \mathcal{T}$
    \EndIf

    \bavtblock{bavtgreen}{4. Backstop Answer Generation}
    \If{$\mathcal{A} = \emptyset \land (b_{\text{tool}, t} = 0 \lor b_{\text{token}, t}/B_{\text{token}} \le \eta)$}
        \State Synthesize answer $n^*_{\text{ans}}$ from $n^* = \arg\max_{n \in \mathcal{N}_{\text{cand}}} V(n)$ 
        \State $V(n^*_{\text{ans}}) \leftarrow V(n^*)$; $\mathcal{A} \leftarrow \mathcal{A} \cup \{n^*_{\text{ans}}\}$; $\mathcal{T} \leftarrow \mathcal{T} \cup \{n^*_{\text{ans}}\}$
    \EndIf
\EndWhile
\State \Return $\hat{y}$ extracted from optimal leaf $n^* = \arg\max_{n \in \mathcal{A}} V(n)$
\end{algorithmic}
\end{algorithm}

\section{Implementation Details and Reproducibility}

\subsection{Hyperparameter Configurations}
To ensure full reproducibility of the Budget-Aware Value Tree (BAVT) framework, we detail the core hyperparameters utilized across all experiments in Table~\ref{tab:hyperparams}. The dynamic search parameters were held constant across all datasets to demonstrate the generalization of the framework.

\paragraph{Value Scaling and Normalization.} 
To maintain stable exponentiation during the budget-annealed node sampling, we enforce a strict bounding on all state values. The prompt-based critic initially estimates the raw state value on a discrete scale from 1 to 10. This raw score is subsequently normalized by dividing by 10, constraining the final operational state value $V(n) \in [0.1, 1.0]$. Furthermore, the residual information delta ($\Delta_t$) predicted by the critic at each step is strictly clipped to the range $[-4, +4]$ before being applied to the parent node's score.

\begin{table}[h]
\centering
\caption{Core hyperparameters for the BAVT framework and LLM generation.}
\label{tab:hyperparams}
\begin{tabular}{lc}
\toprule
\textbf{Parameter} & \textbf{Value} \\
\midrule
Raw Critic Value Scale & $[1, 10]$ \\
Residual Delta Bounds ($\Delta_t$) & $[-4, +4]$ \\
Normalized State Value Range $V(n)$ & $[0.1, 1.0]$ \\
Terminal Confidence Threshold ($\tau$) & 0.8 \\
Budget Backstop Threshold ($\eta$) & 0.2 \\
Max Output Tokens per Call & 512 \\
\bottomrule
\end{tabular}
\end{table}

\subsection{Cost Estimation Analysis}

To contextualize the economic efficiency of the BAVT framework during deployment, we estimate the maximum theoretical cost per sample across our three predefined budget configurations.

\paragraph{Pricing and Assumptions.}
For our estimation, we assume the token generation budget ($B_{\text{token}}$) represents the maximum output tokens. In multi-hop reasoning tasks with retrieved contexts, input prompts are significantly longer than generated outputs; therefore, we conservatively estimate input token volume to be $10\times$ the output token volume. Token costs are calculated using standard API pricing per 1M tokens\footnote{API pricing rates were retrieved from \url{https://deepinfra.com}.}: GPT-OSS-20B (Reasoning Model) at \$0.03 input / \$0.14 output, and Qwen3-30B-A3B (Instruct Model) at \$0.08 input / \$0.28 output. The external search API is priced at a flat rate of \$0.005 per query.

As detailed in Table~\ref{tab:cost_estimation}, the analysis reveals a stark economic disparity between LLM inference and tool utilization. Across all budget tiers and model architectures, the cost of executing external search queries vastly dominates the total expenditure, consistently accounting for over 90\% of the cost per sample. This underscores the critical necessity of the BAVT step-level critic, which aggressively prunes redundant or uninformative tool calls to maximize budget efficiency in real-world deployments.

\begin{table}[h]
\centering
\caption{Estimated maximum cost per sample (in USD) across budget tiers. Token costs assume an input-to-output token ratio of 10:1. Search operations constitute the vast majority of the total cost.}
\label{tab:cost_estimation}
\resizebox{\textwidth}{!}{
\begin{tabular}{lcccccc}
\toprule
\textbf{Budget} & \textbf{Max Tool} & \textbf{Search} & \multicolumn{2}{c}{\textbf{Reasoning Model (GPT-OSS)}} & \multicolumn{2}{c}{\textbf{Instruct Model (Qwen)}} \\
\cmidrule(lr){4-5} \cmidrule(lr){6-7}
\textbf{Tier} & \textbf{Calls} & \textbf{Cost (\$)} & \textbf{Token Cost (\$)} & \textbf{Total Cost (\$)} & \textbf{Token Cost (\$)} & \textbf{Total Cost (\$)} \\
\midrule
Low & 5 & 0.025 & 0.00088 & 0.02588 & 0.00108 & 0.02608 \\
Middle & 10 & 0.050 & 0.00176 & 0.05176 & 0.00216 & 0.05216 \\
High & 20 & 0.100 & 0.00352 & 0.10352 & 0.00432 & 0.10432 \\
\bottomrule
\end{tabular}
}
\end{table}

\section{System Prompt Templates}
\label{sec:appendix_prompts}

To ensure complete reproducibility of the BAVT framework, we provide the exact system prompts used in the
planning, generation, and evaluation phases.

\subsection{Generator (Actor) System Prompt}

\begin{PromptBox}{System Prompt: BAVT Generator}
You are a precise AI assistant for multi-hop question answering.

Rules:
- Think step by step and explicitly show your reasoning.
- In each turn, do exactly one action: either call one tool or provide the final answer.
- You may use tools multiple times and reason across multiple rounds if needed.
- Submit exactly ONE final answer.
- The final answer must be strictly one line, with no extra text and no newline.

Final Answer Output Format:
<answer>FINAL_ANSWER</answer>

{dynamic_instruction}
\end{PromptBox}

\subsection{Step-Level Critic Prompt}

\begin{PromptBox}{System Prompt: BAVT Critic}
You are a strict evaluator. Your job is NOT to solve the problem, but to judge the quality of the Agent's latest action and score it.

### Value Score (Delta)
- The absolute value is in the range [1, 10]. Node 0 starts at 1.
- 10 means a correct final answer is certain given current information.
- 1 means a correct final answer is impossible given current information.
- You will be given the previous absolute value and the full step history (actions, observations, values).
- Output ONLY the delta relative to the previous value (e.g., +1, -1, 0). Do NOT output an absolute value.
- The delta MUST be an integer in [-4, +4].
- Keep deltas conservative; if evidence is weak/ambiguous, avoid positive deltas. If no meaningful gain, use <= -1.
- Ensure the implied new value stays within [1, 10].

### Score Calibration (5 tiers; be conservative)
Score meaning (absolute value in [1,10]): how likely a fully correct final answer is, given current grounded info.
- 1-2 (Tier 1: No basis / impossible): almost no grounded support, major contradictions, or blocked path.
- 3-4 (Tier 2: Weak / early): some relevant signals, but key facts missing; multiple plausible answers remain.
- 5-6 (Tier 3: Partial / uncertain): several key facts covered, but at least one major gap/uncertainty remains.
- 7-8 (Tier 4: Strong / likely): most essential facts covered; only small residual uncertainty.
- 9-10 (Tier 5: Near-certain / certain): all essential facts directly supported; no plausible alternative. 10 should be rare.

### Strict Evaluation Criteria (be harsh)
1) Correctness and verifiability of observations (penalize speculation or errors).
2) Direct relevance to the user question and current plan.
3) Information gain vs. repetition (no credit for redundant steps).
4) Coverage and completeness progress toward a full answer.
5) Reasoning quality (no unjustified leaps from evidence).
6) Efficiency vs. information gain (low gain for repeated effort is negative).
7) Risk and uncertainty (if ambiguity remains, avoid positive deltas).
8) If uncertain, choose deepen.

### Output Format
Output ONLY a JSON object (no Markdown, no extra text):
{
  "delta": 0
}
Do NOT include any other fields.
\end{PromptBox}

\subsection{High-Level Planner Prompt}
\label{subsec:planner_prompt}

\begin{PromptBox}{System Prompt: BAVT Planner}
You are a high-level Multi-hop QA planner. Provide an outline plan only; do NOT answer the question.
- This conversation is planning-only. Never provide a final answer to the user question.
- Do not invent facts, entities, names, dates, or sources.
- Keep the plan abstract: do NOT list specific searches, tools, websites, queries, or data sources.
- Keep to 2-5 steps. Can also give alternative paths.
- Keep each line $\le$ 120 characters.
- State what must be established at each hop (the kind of fact/relationship), not how to retrieve it.

{budget_hint}

Question:
{question}

Follow this template to give a high-level plan:
This turn is planning-only: do NOT provide the final answer.
- Hop 1: <what link must be established>
- Hop 2: <next link>
...
Estimated tool calls: <how many tools will be called to answer the question>
Stop condition: <when the plan is sufficient to produce the final answer>
\end{PromptBox}

\subsection{Dynamic Structural Instructions}
During tree expansion, the placeholder \texttt{\{dynamic\_instruction\}} in the Generator prompt is replaced by one of the following deterministic commands, based on the Critic’s value estimates:

\begin{itemize}[leftmargin=1.2em]
  \item \textbf{Answer:}
  \begin{PromptBox}{Dynamic Instruction: Answer}
Instruction: Answer now. Do NOT call any tools. Return exactly one line in this format:
<answer>FINAL_ANSWER</answer>.
  \end{PromptBox}

  \item \textbf{Widen ($V(s_t) \le V(s_{t-1})$):}
  \begin{PromptBox}{Dynamic Instruction: Widen}
Instruction: This step did not produce a positive value gain. Try a different query or line of reasoning. Avoid repeating the previous path and do not answer yet.
  \end{PromptBox}

  \item \textbf{Deepen ($V(s_{t-1}) < V(s_t) < \tau$):}
  \begin{PromptBox}{Dynamic Instruction: Deepen}
Instruction: Deepen with one next step. Do exactly one action: if evidence is already sufficient, answer now with <answer>...</answer> and do NOT call tools; otherwise call exactly one tool.
  \end{PromptBox}

  \item \textbf{Budget Backstop (Resources Exhausted):}
  \begin{PromptBox}{Dynamic Instruction: Budget Backstop}
Instruction: You MUST ANSWER NOW and you must not call tools. You may include minimal analysis, but your final answer must appear in a <answer>FINAL_ANSWER</answer> tag.
  \end{PromptBox}
\end{itemize}

\end{document}